\documentclass[12pt, reqno]{amsart}
\usepackage{geometry}
\usepackage{amsthm}
\usepackage{amssymb}
\newtheorem{definition}{Definition}[section]
\newtheorem{theorem}{Theorem}[section]
\newtheorem{lemma}{Lemma}[section]
\usepackage{enumerate}
\usepackage{graphicx}
\graphicspath{{./images/}}
\usepackage[T1]{fontenc}
\usepackage[utf8]{inputenc}
\numberwithin{equation}{section}
\usepackage[nomarkers]{endfloat}
\usepackage{hyperref}

\let\emptyset\varnothing

\title{`Almost Sure' Chaotic Properties of Machine Learning Methods }
\author{Nabarun Mondal}
\address{D.E.Shaw \& Co. India, Hyderabad }
\email{mondal@deshaw.com}
\thanks{Nabarun Mondal : 
Dedicated to my missing geometry teacher Dr. Sushanta Mondal;
my parents : Tapan and Sabita Mondal;
Big thanks to : Abhishek Chanda and Shweta Bansal : You all have been constant support.
}

\author{Partha P. Ghosh}
\address{Microsoft India, Hyderabad }
\email{parthag@microsoft.com}
\thanks{ Partha. P. Ghosh : 
Dedicated to my parents and family, without their presence we are nothing. \\
}


\begin{document}


\begin{abstract}
It has been demonstrated earlier that universal computation is `almost surely' chaotic.
Machine learning is a form of computational fixed point iteration, iterating over the
computable function space. We showcase some properties of this iteration,
and establish in general that the iteration is `almost surely' of chaotic nature. 
This theory explains the observation in the counter intuitive properties of deep learning  
methods. This paper demonstrates that these properties are 
going to be universal to any learning method. 

\end{abstract}

\maketitle

\begin{section}{Motivation}\label{intro}

The motivation of the current paper is two fold.

One of the authors of the current paper was using iterative machine learning to crack cipher code in late 1990s.
While doing so he came to the astounding realisation that the resulting learned function is not \emph{convergent} (definition \ref{cs}) at all.
What is really meant is that : data points $x_n$ which were accepted up-to iteration $n$ ( points in the set $S_n$ )
would fly away eventually after some more iteration at $m$ ($S_m$ with $m > n$): 
$$
\not \exists n > N \text{ such  that } S_{n+1} \subseteq S_n   \text{ and } S_n \subseteq S_{n+1}
$$
In other words none of the \textbf{lim sup} $S_n$ and \textbf{lim inf}  $S_n$ exists and therefore :
$$
\not \exists S = \lim_{n \to \infty } S_n   
$$

In essence the iteration was showing converging and then sudden diverging behaviour.
While he communicated this finding to one of pioneers in the image processing domain - he was simply baffled.
In fact the trivia is - he simply remarked : \emph{``I have no idea - what to respond!''}.
While this behavior was never fully understood by the author then, recent theory by the same author(s) \cite{gm} 
seems to have an explanation to it.

A very recent (yet to be published) paper in arxiv \cite{gfc} discusses 
interesting properties in the deep learning. It became immediately apparent to the authors that 
these two phenomenons are essentially connected. 

This paper tries to explain these phenomenons in the light of the discovery of 
chaos \emph{almost surely happening in computation} (theorem \ref{uc-chaos}) \cite{gm}.
\end{section}

\begin{section}{The Nature of Learning Theory}\label{nlt}
Learning theory is really about classification.
As Vapnik \cite{vv} pointed it out - it is all about finding a classifier function $f : X \to \{ 0 , 1\} $ such as to 
isolate positive samples ( accept ) from negative ones ( reject ), aided by training set $T_S$ ; which is
a set of ordered pairs $T_S = \{ (x,y) \} $ with $x \in X$ and $y \in \{ 0, 1\}$.
Given $x \in A$ then $ y = 1$ else $y = 0$, with $A \subset X$ being the accept set.

\begin{subsection}{Learned Function as Indicator of Accept Set}
We can say that to learn $A$ one needs to isolate the set $A \subset X$ using training data sets $T_S$. 
We need to learn the function $f_L$ with $x \in X$ :
$$
f_L(x,T_S) = \left\{ \begin{array}{rl}
 1 &\mbox{ if $x \in A$  } \\
  0 &\mbox{ if $x \in A^C$   }
       \end{array} \right.
$$

Clearly then, the function we want to learn, is the indicator function \cite{hlr} of the accept set $A$,
that is:
$$
f_L(,TS) \equiv 1_A 
$$ 

Then we can simply suggest that machine learning is all about finding the positive samples ( accept ) sets indicator function $1_A$.
This formulation has a problem, because we really do not know that whether or not $1_A$ is computable.
The set of computable functions in this paper would be designated as $\mathbb{C}$.
If $1_A \not \in \mathbb{C}$ we assume existence of a sequence of computable functions $<C_n>$ which converges \cite{hlr} to $1_A$. 
$$
\lim_{n \to \infty } C_n = 1_A
$$
This indeed is a very strong axiom - and in the very next section \eqref{prob-chaos-lt} 
we shall see that this almost surely not holding to be true.
\end{subsection}

\begin{subsection}{Learning theory as Iterative Maps}\label{lt-itr}

We need to generate this sequence of functions $<C_n>$ in a computable way.
That immediately gives :
$$
C_{n+1} = \mathcal{L}(C_n) 
$$
the guiding equation of the learning. There has to be a fixed point computable iteration 
function $\mathcal{L}$  such that it can take a computable function, returns another one, 
such that the orbit  (definition \ref{orbit}) $<C_n>$ eventually converges to ``$1_A$''.

In effect what we seek is :
$$
\lim_{n \to \infty } C_n = 1_A  = \mathcal{L}(1_A) 
$$
with  $1_A$ is the fixed point of the iteration (definition \ref{fp}). 
All ML and Deep Learning methods can be abstracted in this form,
it is in effect a simple tautology. 

NOTE however that the function $\mathcal{L} : \mathbb{Q} \to \mathbb{Q}$ iteration 
might not be able to reach  $1_A$ because in general the set of functions $f \in \mathbb{R}^\mathbb{R}$,
which would be discussed in the next section. 
\end{subsection}

\begin{subsection}{The Cut Function Approximation of $1_A$ }
It is a decision that whether or not $x \in X$ with $1_A(x) = 1$ when $x \in A$ 
and $1_A(x) = 0 $ when $x \not \in A$.
Therefore, clearly $1_A$ partitions the input space $X$. The question of mechanism of 
computable partitioning requires a definition.
\begin{definition}\label{bdt}
\textbf{Binary Decision Tree.}

A binary tree with each nodes having two decider Turing Machine (definition \ref{decider}) 
from a finite set of decider machines $(D_f,D_d ) \in \mathbb{D}$ such that:
$$
D_f : \mathbb{Q} \to \mathbb{Q} \text { and } D_d : \mathbb{Q} \to \{0,1,l,r\}  
$$
Given $x \in X$ with $X \subseteq \mathbb{Q}$ is the input to the current node, let : $y = D_f(x) $ and $d = D_d(y) $.
If $d=0$ the input is rejected and the system is halted.
If $d=1$ the input is accepted and the system is halted.
If $d = l$ then $y$ is passed as input to the left child ; if no left child exists, reject input x and halt.
If $d = r$ then $y$ is passed as input to the right child ; if no right child exists, reject input x and halt.
\end{definition}

This structure (definition \ref{bdt}) partitions the input space into finite number of equivalent partitions such that:
$$
X = \bigcup_i P_i
$$
with either $ P_i  \subset A$ or $ P_i  \subset X \setminus A$, but can not be both. 
We also note that if $P_i \subset A$ then $P_{i+1} \subset X \setminus A$. 
Clearly then $A$ is a disjoint union of sets:
$$
A = \bigcup a_i \; ; \; a_i \cap a_j = \varnothing \; ; \; i \ne j 
$$

\end{subsection}

\end{section}

\begin{section}{Probability, Chaos and Learning Iterations}\label{prob-chaos-lt}
We demonstrate that the general iterative learning system will be \emph{`almost surely'} \cite{wf}\cite{wd} chaotic \cite{cfnfs} \cite{gm}.

\begin{subsection}{Properties of the The Learning Iteration}

We need some definitions to formally present the ideas discussed before. 

\begin{definition}\label{rcf}
\textbf{Rational Mapping of a Computable Function.}

Given an universal turing machine $\mathbb{U}$ (definition \ref{UTM}) any computable function can be encoded using the symbols
in the tape of the machine. The rationalisation (definition \ref{rat}) of the tape $\rho(T_C) $ then, 
serves as the rationalisation of the computable function.
$$
\rho(C) = \rho(T_C, \mathbb{U} ) = c \; ; \; C \in \mathbb{C}
$$
where $\mathbb{C}$ is the set of computable functions.
This makes $c \in [0,1] \cap \mathbb{Q}$.  
\end{definition}

\begin{definition}\label{cf}
\textbf{The Computable Functional.}

A computable function $\mathcal{L} : \mathbb{Q} \to \mathbb{Q} $ is called a functional iff
the range of the function can be interpreted as an encoding of a computable function. 
That is :
$$
\forall x \; ; \; \exists C \in \mathbb{C} \; ; \; \rho(C) = \mathcal{L}(x) 
$$
\end{definition}

\begin{definition}\label{li}
\textbf{The Learning Iteration.}

Given a computable functional $\mathcal{L} : \mathbb{Q} \to \mathbb{Q}$, and input $T_S$
(training set) the learning iteration is defined as :
$$
c_{n+1} = \mathcal{L}(c_n)
$$    
with $c_0 = \rho(T_S)$.
Here $c_n$ signifies the rationalisation (definition \ref{rcf}) of the $C_n \in \mathbb{C}$. 
\end{definition}

\begin{theorem}\label{lic}
\textbf{Learning Iteration is Almost Surely Chaotic.}

Learning iteration, as defined as definition \eqref{li} is almost surely chaotic.
\end{theorem}
\begin{proof}
We note : $\mathcal{L} : \mathbb{Q} \to \mathbb{Q} $, from (theorem \ref{fpi-chaos}) the theorem on chaos on computation, 
that almost every rational sequence is chaotic, the result is immediate.  
\end{proof}

\begin{theorem}\label{licr}
\textbf{Almost Sure Non-Computability for converging function.}

The iteration of definition \eqref{li}, when converges, almost surely converges to an un-computable function.
\end{theorem}

\begin{proof}
This is trivial from Real analysis.
We know that the sequence $c_n \in \mathbb{Q}$, so to complete the space (definition \ref{cncms} ) 
we need to have  the embedding space  $\mathbb{R}$. 
Clearly then, almost all limit points would be irrational 
( actually transcendental ), with a cardinality equal to continuum $|\{ lim<c_n> \}| = \aleph_1$
which is a well known theorem from real analysis \cite{hlr}.
Irrational numbers take infinitely many symbols to encode, 
therefore, the function limit $lim<c_n>$ can not be encoded in a Turing machine tape, 
and hence, is obviously non computable.
 
This is precisely what we wanted to show.
\end{proof}

\end{subsection}

\begin{subsection}{Properties of Self Similarity in Decision Tree}
We start with a definition of \emph{self similarity}.

\begin{definition}\label{ss}
\textbf{Self Similarity.}

Let there be a topological space $X$ (definition \ref{top-space}) and a set of non-surjective homeomorphic functions $\{ f_s\}$
(definition \ref{hom}) indexed by a finite index set $S = \{ s \}$ with :
$$
X = \bigcup_{s \in S} f_s(X) \; .
$$
If $X\subset Y$, we call $X$ self-similar if it is the only non-empty subset of $Y$ 
such that the equation above holds for.
Then $\mathcal{S} = ( X , S , \{ f_s \} )$ is called a self similar structure.
\end{definition}

\begin{theorem}\label{idtss}
\textbf{Infinite Binary Decision Trees have Self Similar Structure.}

A binary decision tree as defined in \eqref{bdt} is called infinite
if it has countably infinite nodes. 
Accept set of such a tree exhibits self similar structure ( definition \ref{ss} ).
\end{theorem}
\begin{proof}
We note that every node in the tree first does a transform 
of the input space $X$ using $D_f$ into $Y$ ; both homeomorphic to $\mathbb{Q}$. 
After that, this space $Y$ is partitioned using 
$D_d$ which are finite in number as discussed in section \eqref{nlt}.
Then this individual partitions are homeomorphic to $\mathbb{Q}$.
Suppose $n(D_d)$ defines the number of partitions. We can then say there are $n(D_d)$
numbers of  homeomorphic (on $\mathbb{Q}$) non-surjective function available for each $(D_d, D_f)$.
This is due to lemma \eqref{p-hom}.

The set of such functions is finite because $\forall n \; ; \; n(D_d)$ and $\mathbb{D}$ is finite.
Suppose then, the set of such composition is termed as $\mathbb{F}$.
It is then trivial that when $X = \mathbb{Q}$ (at the root) then : 
$$
\mathbb{Q} = \bigcup_{F \in \mathbb{F} } F( \mathbb{Q} )
$$
Therefore, infinite binary decision tree have self similar structure. 
\end{proof}
In fact it is well known that \emph{Cantor Set} \cite{hjss}\cite{gc}\cite{cfnfs} is a generalisation 
of this sort of structure ( a dyadic Tree ). 

\begin{theorem}\label{cliss}
\textbf{Convergent ML Function would have Self Similar Structure.}

If an ML iteration (definition \ref{li}) converging to a function $f_A$,
then accept set of $f_A : A$ almost surely would have a self similar structure. 
\end{theorem}
\begin{proof}
Almost surely the function is un-computable (theorem \ref{licr}). 
We note that the due to the fixed size of the learner algorithm, the number of deciders
the learning system can stays finite. 
And therefore, to be convergent the structure is to be extended to infinity :
that is an infinite binary decision tree (definition \ref{bdt}).
Now, using theorem \eqref{idtss}, the result is immediate. 
\end{proof}

\end{subsection}

\end{section}

\begin{section}{Chaos And Machine Learning - A Summary }\label{chaos-rs}
The theorems proven in the sections before can be useful to deduce very interesting properties of machine learning,
which clearly showcases the problems arising from the chaotic nature of the Universal Computation.
\begin{enumerate}

\item{The first phenomenon - that non converging ML process experienced by the co-author (section \ref{intro}) 
is clearly explained by the  inherent chaotic (theorem \ref{lic}) nature of the ML iteration. 
The functions, encoded in the rational space, almost surely 
has a chaotic orbit, and that is why output of function $C_n$ can drastically differ from $C_{n+1}$, given same input.
However, we very well know that although almost all numbers are normal and transcendental proving them so is a problem, 
such is the problem here. Proving that the specific behaviour is chaotic can be done only in case to case basis.
}

\item{ For the second phenomenon reported \cite{gfc} : 
As stated clearly in \cite{gfc} that same ML algorithm from a different subset of the original training set facing
the same problem.
This is happening because in this case ML iteration is actually converging, 
and generating a fractal partition of the input space (theorem \ref{cliss}).
This is precisely what they found : \emph{input to output mapping mostly discontinuous} .
They also noted that for each input which gets accepted, there are dense (definition \ref{dense-set}) 
set of inputs which gets rejected. 
This is a standard property for fractal space \cite{cfnfs}.

So, to summarise when the iteration converges, the result would `almost surely' be a fractal set.
Iteration from different initial conditions would, in fact converge to different indicator function, hence
would accept different fractal sets $A_1,A_2$ . Albeit $A_1 \cap A_2 \ne \varnothing$ but also $A_1 \Delta A_2 \ne \varnothing$
in fact the set $A_1 \Delta A_2 \ne \varnothing$ should be dense (definition \ref{dense-set}) in $X$. 

This clearly demonstrates the chaotic nature of the ML iteration convergence.      
}

\end{enumerate}

Finally we end with the same note they have : \emph{indeed, any form of computable machinery would exhibit behaviour of this kind.
These chaotic behaviours are in effect Universal}, as clearly demonstrated in this paper.    
\end{section}

\appendix
\begin{section}{Definitions Used}\label{ap_1}
\begin{definition}\label{fp}
\textbf{Fixed Point of a function. }

For a function $f:X \to X$ , $x^*$ is said to be a fixed point, iff $f(x^*) = x^*$ .
\end{definition}

\begin{definition}\label{mp}
\textbf{Metric Space.}

A metric space is an ordered pair $(M,d)$ where $M$ is a set and $d$ is a metric on $M$ , i.e., a function:-

$$
d : M \times M \to \mathbb{R}
$$

such that for $x,y,z \in M$ , the following holds:-
\begin{enumerate}
\item { $d(x,y) \ge 0 $ }
\item { $ d(x,y) = 0 $ iff $x=y$ . }
\item { $d(x,y) = d(y,x) $ }
\item { $d(x,z) \le  d(x,y) + d (y,z)$ }
\end{enumerate}

The function `$d$' is also called ``distance function'' or simply ``distance''.
\end{definition}

\begin{definition}\label{cs}
\textbf{Cauchy Sequence in a Metric Space $(M,d)$ .}

Given a Metric space $(M,d)$ , the sequence $x_1,x_2,x_3,...$ of real numbers is called `Cauchy Sequence', if for every positive real number $\epsilon$ , 
there is a positive integer $N$ such that for all natural numbers $m,n > N$ the following holds:-
$$
d (x_m , x_n ) < \epsilon .
$$

\end{definition}
Roughly speaking, the terms of the sequence are getting closer and closer together in a way 
that suggests that the sequence ought to have a limit $x^*  \in M$ . Nonetheless, such a limit does not always exist within $M$ .

Note that by the term: \emph {sequence} we are implicitly assuming \emph {infinite sequence} , unless otherwise specified.

\begin{definition}\label{cncms}
\textbf{Complete Metric Space.}

A metric space $(M,d)$ is called complete (or Cauchy) 
iff every Cauchy sequence (definition \ref{cs}) of points in $(M,d)$ has a limit , that is also in $M$ .
\end{definition} 

As an example of not-complete metric space take $\mathbb{Q}$ , the set of rational numbers. 
Consider for instance the sequence defined by $x_1 = 1$ and function $d$ is defined by standard
difference between $d(x,y) = |x-y|$ , then :-
$$
x_{n+1} = \frac{1}{2} \left ( x_n + \frac{2}{x_n} \right ) 
$$

This is a Cauchy sequence of rational numbers,
but it does not converge towards any rational limit, but to
$\sqrt{2}$ , but then $\sqrt{2} \not \in \mathbb{Q}$ . 

The closed interval $[0,1]$ is a Complete Metric space which is homemorphic (definition \ref{hom}) to $\mathbb{R}$.

\begin{definition}\label{orbit}  \cite{ap} \cite{mbgs}
\textbf{Orbit.}

Let $f:X \to X$ be a function. 
The sequence $ \mathcal{O} = \{x_0, x_1,x_2,x_3,...\}$ where
$$
x_{n+1} = f(x_n) \; ; \; x_n \in X \; ; \; n \ge 0 
$$
is called an orbit (more precisely `forward orbit') of $f$. 

$f$ is said to have a `closed' or `periodic' orbit $ \mathcal{O}$ if $| \mathcal{O}| \ne \infty$ .
\end{definition}

\begin{definition}\label{top-space}
\textbf{Topological Space.}

Let the empty set be written as : $\emptyset$. Let $2^X$ denotes the power set, i.e. the set of all subsets of $X$.
A topological space is a set $X$ together with $\tau \subseteq 2^X$ satisfying the following axioms:-
\begin{enumerate}
\item{ $\emptyset \in \tau$ and $X \in \tau$ ,}
\item{ $\tau$ is closed under arbitrary union, }
\item{ $\tau$ is closed under finite intersection. }
\end{enumerate}
The set $\tau$ is called a topology of $X$.
\end{definition} 

\begin{definition}\label{dense-set}
\textbf{Dense Set.}

Let $A$ be a subset of a topological space $X$. 
$A$ is dense in $X$ for any point $x \in X$, if any neighborhood of $x$ contains at least one point from $A$.
\end{definition}

The real numbers $\mathbb{R}$ with the usual topology have the rational numbers $\mathbb{Q}$ as a countable dense subset.

\begin{definition}\label{hom}
\textbf{Homeomorphism.}

A function $f: X \to Y$ between two topological spaces $(X, T_X)$ and $(Y, T_Y)$ is called a homeomorphism if it has the following properties:
\begin{enumerate}
\item{ f is a bijection (one-to-one and onto),}
\item{f is continuous,}
\item{the inverse function $f^{-1}$ is continuous (f is an open mapping).}
\end{enumerate}

\end{definition}

\begin{lemma}\label{p-hom}
\textbf{Existence of Homeomorphic functions on Partitions.}

Let $\mathbb{P} = \{ P_i \}$ be a partition of $X$ homeomorphic to $\mathbb{Q}$, such that:
$$
\bigcup_i^n P_i = X \; ; \; \forall i \ne j \; P_i \cap P_j = \varnothing  
$$     
then, there exists $n = |\mathbb{P}|$ homeomorphic functions from $X \to P_i$.
\end{lemma}

\begin{definition}\label{bounded-seq}
\textbf{ Bounded Sequence.}

A sequence $<x_n>$ is called a bounded sequence iff :-
$$
\forall n \; \; l \le x_n \le u \; ; \;  -\infty < l  \le  u < \infty  
$$

The number `l' is called the lower bound of the sequence and 
`u' is called the upper bound of the sequence.
\end{definition}

\begin{lemma}\label{bcss}
\textbf{Bolzano-Weierstrass.}

Every bounded sequence has a convergent (Cauchy) subsequence.
\end{lemma}

It is to be noted that a bounded sequence may have many convergent subsequences (for example, a sequence consisting of a counting of the rationals has subsequences converging to every real number) or rather few (for example a convergent sequence has all its subsequences having the same limit).

\begin{definition}\label{TM}
\textbf{Turing Machine.}

A ``Turing Machine'' is a 7-tuple ($Q,\Sigma,\Gamma,\delta,q_0,q_a,q_r$), where:-
\begin{enumerate}
\item{ $Q$ is the set of states. }
\item{ $\Sigma$ is the set of input alphabets not containing the blank symbol $\beta$. }
\item{ $\Gamma$ is the tape alphabet , where $\beta \in \Gamma$ and $\Sigma \subseteq \Gamma$. }
\item{ $\delta : Q \times \Gamma \to Q \times \Gamma \times \{L , R \} $ is the transition function. }
\item{ $q_0 \in Q$ is start state.}
\item{ $q_a \in Q$ is the accept state.}
\item{ $q_r \in Q$ is the reject state.}
\end{enumerate}
\end{definition} 
According to standard notion $q_a \ne q_r$ , but we omit this requirement here, 
as we are not going to distinguish between two different types of halting (`accept and halt' vs `reject and halt') of Turing Machines.

A Turing Machine `$M$' (definition \ref{TM}) computes as follows.

Initially `$M$' receives the input $w=w_1w_2w_3...w_n \in \Sigma^* $ 
on the leftmost `$n$' squares on the tape, and the rest of the tape is filled up with blank symbol `$\beta$'.
The head starts on the leftmost square on the tape. 
As the input alphabet `$\Sigma$' does not contain the blank symbol `$\beta$', 
the first `$\beta$' marks end of input.

Once `$M$' starts, the computation proceeds wording to the rules of `$\delta$'.
However, if the head is already at the leftmost position, then, 
even if the `$\delta$' rule says move `$L$' , the head stays there.

The computation continues until the current state of the Turing Machine is either $q_a$ , or $q_r$ .
In lieu of that, the machine will continue running forever.

\begin{definition}\label{decider}
\textbf{Decider Turing Machine.}

A Turing Machine, which is guaranteed to halt on any input (i.e. reach one of the states \{$q_a,q_r$\} ) is called a decider.
\end{definition}

\begin{definition}\label{undecidable}
\textbf{Undecidable Problem.}

If for a given problem, it is impossible to construct a  decider (definition \ref{decider}) Turing Machine, 
then the problem is called undecidable problem.
\end{definition}

\begin{definition}\label{UTM}
\textbf{Universal Turing Machine.}

An `UTM' or `Universal Turing Machine' is a Turing Machine (definition \ref{TM}) such that it can simulate an 
arbitrary Turing machine on arbitrary input.
\end{definition}

\begin{definition}\label{CTT}
\textbf{Church Turing Thesis.}

Every effective computation can be carried out by a Turing machine (definition \ref{TM}), 
and hence by an Universal Turing Machine(definition \ref{UTM}).
\end{definition}

\begin{definition}\label{god}
\textbf{G\"{o}delization (G\"{o}del).}

Any string from an alphabet set $\Gamma$ can be represented as an integer in base `$b$' with $b = |\Gamma|$. 
To achieve this, create a one-one and onto G\"{o}del map $g : \Sigma \to D_b$ , where,
$$
D_b = \{ 0 , 1, 2, ... ,b-1 \} .
$$ 
G\"{o}delization or $\mathbb{G} : \Sigma^+ \to \mathbb{Z_+}$ then, is defined as follows:
 
A string of the form $w = w_{n-1}w_{n-2}...w_1w_0$ , with  $w_i \in \Gamma$ ,
can be mapped to an integer $I_w = \mathbb{G}(w)$ as follows \cite{gm}:
$$
I_w = \mathbb{G}(w) = \sum\limits_{k=0}^{n-1} g(w_k) b^{k}
$$

\end{definition}
The common decimal system is a typical example of G\"{o}delization of symbols from $\{ 0,1,..,8,9\}$.
The binary system represents G\"{o}delization of symbols from  $\{ 0,1\} $.
As a far fetched example, any string from  the whole english alphabet, can be written as a base 26 integers!
 
\begin{definition}\label{rat}
\textbf{Rationalization. }

Any string `$w$' of length `n' ($|w|=n$) , created from an alphabet set $\Gamma$,
can be represented as a rational number $x \in \mathbb{Q}$.
We define the rationalization, $\rho$ , in terms of G\"{o}delization (definition \ref{god}) as follows \cite{gm}:
$$
x = \rho(w) = \frac{\mathbb{G}(w)}{ b^n } = \mathbb{G}(w) b^{-n} = 0.w_{n-1}w_{n-2}...w_0
$$
By definition, $x \in [0,1] \cap \mathbb{Q}$.
\end{definition}

\begin{theorem}\label{fpi-chaos} 
\textbf{Bounded non repeating sequences are chaotic.}

Suppose $<x_n>$ is an infinite sequence such that $x_n \in \mathbb{Q}$,
and $x_i \ne x_j $ when $i \ne j$, then $<x_n>$  is chaotic.
Given a bounded sequence, it is almost surely chaotic \cite{gm}.  

\end{theorem}   

\begin{theorem}\label{uc-chaos} 
\textbf{Universal Computation is `Almost Surely' chaotic.} 

The rationalization of the sequences $<T_n>$ that is $<\rho(T_n)>$
of a Universal Turing machine is a bounded sequence between $[0,1]\cap \mathbb{Q}$ 
and hence `almost surely' chaotic (theorem \ref{fpi-chaos}) \cite{gm}.  
\end{theorem}

\end{section}

\end{document}